\documentclass{article}

\pdfoutput=1
\usepackage{microtype}
\usepackage{graphicx}
\usepackage{subfigure}
\usepackage{booktabs} 
\usepackage{mathtools, nccmath}
\usepackage{amsfonts}
\usepackage[font=small,labelfont=bf]{caption} 
\usepackage{enumerate}
\DeclareMathOperator{\Llocal}{L_\textup{local}}
\DeclareMathOperator{\Lclassify}{L_\textup{classify}}
\newcommand{\NA}{---}
\newcommand{\tabsplit}[1]{%
  \begin{tabular}[t]{@{}c@{}}   
    #1
  \end{tabular}
  }
\usepackage{natbib}

\usepackage[mathscr]{eucal}
\usepackage[inline]{enumitem}
\usepackage{arxivstyle}
\usepackage[left=1in, top=1in, bottom=1in, right=1in]{geometry}
\usepackage{xcolor}

\usepackage{hyperref}
\usepackage[toc,page]{appendix}

\usepackage{notationextras}

\newif\ifincludenotationtable
\newif\ifarxiv
\newif\ifpaper

\title{Meta-Learning to Cluster}
\date{\vspace{-5ex}}
\author{Yibo Jiang\footnote{Department of Computer Science, Columbia University, New York, US, \texttt{\url{yj2460@columbia.edu}, \url{verma@cs.columbia.edu}}}  \and \addtocounter{footnote}{0} Nakul Verma\footnotemark[\value{footnote}]}
\begin{document}

\maketitle

\begin{abstract}
Clustering is one of the most fundamental and wide-spread techniques in exploratory data analysis. Yet, the basic approach to clustering has not really changed: a practitioner hand-picks a task-specific clustering loss to optimize and fit the given data to reveal the underlying cluster structure. Some types of losses---such as $k$-means, or its non-linear version: kernelized $k$-means (centroid based), and DBSCAN (density based)---are popular choices due to their good empirical performance on a range of applications. Although every so often the clustering output using these standard losses fails to reveal the underlying structure, and the practitioner has to custom-design their own variation. In this work we take an intrinsically different approach to clustering: rather than fitting a dataset to a specific clustering loss, we train a recurrent model that \emph{learns} how to cluster. The model uses as training pairs examples of datasets (as input) and its corresponding cluster identities (as output). 
By providing multiple types of training datasets as inputs, our model has the ability to generalize well on unseen datasets (new clustering tasks).
Our experiments reveal that by training on simple synthetically generated datasets or on existing real datasets, we can achieve better clustering performance on unseen real-world datasets when compared with standard benchmark clustering techniques. Our meta clustering model works well even for small datasets where the usual deep learning models tend to perform worse.
\end{abstract}

\section{Introduction}
Clustering is one of the most ubiquitous techniques in data analysis that has important applications in numerous domains that extend far beyond machine learning \citep{xu2003document,cai2004hierarchical,tung2010enabling,Xu2015}. 
The main objective of clustering is to group a set of objects in such a way that objects in the same group (cluster) are more \emph{similar} to each other (with respect to some domain-specific notion of similarity), than to those in other groups.

The prevailing approach to clustering is to optimize a specific objective function (usually called the cluster loss) that encodes the desired domain-specific notion of similarity to reveal the underlying group structure in data. Certain cluster losses such as $k$-Means and DBSCAN \citep{DBLP:conf/kdd/EsterKSX96} have shown promising results in various applications, such as object recognition \citep{stockman1987object} and recommendation systems \citep{kim2008recommender}. 
In modern settings, as the data we collect has become more complex, more expressive cluster losses encoded by deep network architectures have shown some success as well  \citep{DBLP:conf/icml/HuMTMS17,DBLP:journals/corr/abs-1801-07648}. 


A practitioner usually hand-picks a cluster loss in an ad-hoc manner; they basically choose whichever loss (either a preexisting one, or a newly designed one) that seems to give a satisfactory result. And a priori, it is unclear how this choice can be made in a principled way, or better yet completely bypass this step and directly achieve a good clustering.  
  
In this work, we want to take a step towards this direction. Rather than optimizing for a specific custom-designed loss, we develop a model that \emph{learns} how to cluster. 
Such ``learning to learn'' approach comes under the framework of meta-learning.
 Instead of training on a large amount of data from a single task, meta-learning systems are trained on a large number of (similar) smaller tasks and are used to make predictions on newly available (similar) tasks.
Deep meta-learning systems have shown remarkable success on supervised learning \citep{DBLP:conf/icml/SantoroBBWL16,Mishra2017ASN} and reinforcement learning tasks \citep{DBLP:journals/corr/MirowskiPVSBBDG16,Wang2018PrefrontalCA}.

We propose a simple yet highly effective meta-learning model to solve for clustering tasks. Our model finds the cluster structure directly without having to choose a specific cluster loss for each new clustering problem. 
There are two key challenges in training such a model for clustering:
\begin{enumerate}[label=(\roman*)]
\item since clustering is fundamentally an unsupervised problem, there is a lack of availability of true cluster identities for each training task, and
\item cluster label for each new datapoint depends upon the labels assigned to other datapoints in the same clustering task.
\end{enumerate}

We systematically address each of these issues: (i) We show that trainining on simple synthetically generated datasets or other real world labelled datasets with similar statistics (such as having the same number of dimensions and number of categories) can generalize to real previously unseen datasets. 
(ii) We train sequentially and use a recurrent network (using LSTMs, \cite{hochreiter1997long}) as our clustering model. A recurrent memory model helps assign clustering labels effectively based on data points seen before. 

Our experiments reveal several remarkable observations: even if we train our meta-learning model just on simple synthetically generated data, we can achieve better clustering results on some real-world problems than by using popular preexisting linear and non-linear benchmark cluster losses.
In addition, our meta model trained on labelled real datasets of different distributions can also transfer its clustering ability to unseen datasets.
Moreover, by effectively \emph{learning} how to cluster, unlike centroid-based clustering approaches like $k$-means, our meta-learning model also has the ability to approximate the right number clusters in simple tasks; thereby obviating the need to pre-specify the number of clusters. 
To best of our knowledge, this is the first clustering model using end-to-end deep meta-learning framework.

\section{Related work}
\label{sec:related}
\textbf{Deep Learning for Clustering.} \citet{DBLP:journals/corr/abs-1801-07648} provide an overview of all major frameworks that use deep learning for clustering tasks. Broadly, deep clustering consists of two parts: feature extraction phase and clustering phase. Usually feature extraction is done through an auto-encoder \citep{hinton2006reducing}, which serves as a new representation for the subsequent clustering phase. \citet{DBLP:conf/icml/HuMTMS17} propose an alternate approach using ``self-augmentation'', which encourages the new representation to map the input data close to its augmentation, hence acting as a regularizer.

Though the extracted features can be used directly by standard clustering algorithms, deep learning models usually optimize further over specific clustering losses. \citet{DBLP:conf/icml/YangFSH17}, for instance, propose to optimize over the $k$-means loss, thus encouraging learning $k$-means friendly feature representations.
\citet{DBLP:conf/icml/XieGF16} on the other hand, use a loss based on student t-distribution and can accommodate for soft clusterings. 
\citet{DBLP:conf/icml/XieGF16}, along with \citet{DBLP:conf/icml/HuMTMS17}, further explore information-theoretic losses to achieve good clusterings.

Like many deep architectures, deep clustering models require a large amount of data to train which may not be possible for many clustering problems. This pushes the need for a model that can work well in a data-limited scenario similar to one-shot or a few-shot learning settings in the supervised meta learning and this is where meta clustering can be effective.  

\textbf{Meta Learning.} Meta learning, often referred to as \emph{learning to learn}, is closely related to one-shot or few-shot learning. It has shown promising results in supervised learning.
   In the standard classification case, it can greatly benefit when training data is limited \citep{DBLP:conf/icml/SantoroBBWL16}. In the reinforcement setting, it benefits by training more generalized agents rather than ones that specialize on a restricted domain \citep{Wang2017LearningTR}. 
   

Standard approaches to learn a meta-learning model include defining a distribution over the structure of input data to perform inference \citep{DBLP:conf/cogsci/LakeSGT11,Snell2017PrototypicalNF}, or to use a memory model such as long short-term memory model (LSTM, \citealp{hochreiter1997long}) \citep{DBLP:conf/icml/SantoroBBWL16,Wang2018PrefrontalCA} 
There are several generic gradient-based learning methods developed for meta-learning, such as MAML \citep{Finn2017ModelAgnosticMF} and Reptile \citep{Nichol2018OnFM}. 

%


To best of our knowledge, only a few works focus on using meta-learning framework for doing clustering. 
Closely related works by \citet{DBLP:conf/semcco/FerrariC12} and \citet{DBLP:journals/isci/FerrariC15} estimate which of the preexisting clustering losses works well for a new clustering task. Their approach is therefore limited to losses which the user has to provide. In contrast, we implicitly learn an appropriate loss for the new clustering task. \citet{DBLP:conf/nips/Garg18} pushes the framework further, laying down the theoretical foundations for meta clustering. But like \citet{DBLP:conf/semcco/FerrariC12} and \citet{DBLP:journals/isci/FerrariC15}, \citet{DBLP:conf/nips/Garg18} uses meta-attributes (like computing data covariance) rather than directly from the input. Moreover, it learns binary similarity functions without explicitly returning the clustering and compare  outputs with a simple majority rule. 
Another interesting line of work by \cite{Hsu2018UnsupervisedLV} uses unsupervised learning to improve upon the downstream supervised learning task.
Our work on the other hand, specifically focuses on learning unsupervised clustering, and shows empirical success. 
Notice that both supervised and reinforcement meta-learning require some sort ``guidance" from data in the form of labels (supervised) or rewards (reinforcement). This is fundamentally different from the unsupervised meta learning. By using a recurrent model and doing multiple passes through datapoints, we achieve ``self-guidance" and consequently good performance.

\section{Meta Learning for Clustering}
\includenotationtabletrue
\label{sec:meta}

\subsection{Problem Setup}
\label{sec:meta-setup}

A meta-clustering model $\mathcal{M}$ maps data points to cluster labels. During meta-learning, the model
is trained to be able to adapt to a set of clustering tasks $\{\mathcal{T}_i \}$. At the end of meta-training, $\mathcal{M}$ would produce clustering labels for new test task $\mathcal{T}_{\textup{test}}$.

In our case, each training task $\mathcal{T}_i$ consists of a set of data points $\mathcal{X}^i$ (a dataset) and their associated cluster labels $\mathcal{L}^i$. Hence, each $\mathcal{T}_i = (\mathcal{X}^i, \mathcal{L}^i)$. It is worth noting that 
$\mathcal{X}^i$ and $\mathcal{L}^i$ are themselves partitioned into subsets based on cluster identities. More specifically, $\mathcal{X}^i = \{\mathcal{X}^i_1,\mathcal{X}^i_2,\ldots,\mathcal{X}^i_{K_i} \}$ and $\mathcal{L}^i = \{\mathcal{L}^i_1,\mathcal{L}^i_2,\ldots,\mathcal{L}^i_{K_i} \}$, where $K_i$ is the number of clusters in task $i$. Notice that unlike the supervised learning case, we allow the number of clusters for each task to vary. By introducing this cluster-specific generalization, our clustering model has the flexibility to potentially approximate the number of clusters during test time. Since labeling for a test task is yet to be determined, the structure of the test task $\mathcal{T}_{\textup{test}}$ is different from training and consists of only a set of datapoints $\mathcal{X}_\textup{test}$.

The cluster labels for training can either come from synthetically generated training tasks or by using labelled data from similar application domains. In Section~\ref{sec:experiments}, we will present experiments on training with synthetic data, real labeled data and mixture of the two.


Another important distinction between supervised and the unsupervised case is that any permutation of the labeling does not change the cluster quality. This additional degree of freedom can potentially hinder the model optimization and the efficacy of the final prediction. We circumvent this issue by fixing a permutation during training, thereby limiting the parameter search space and accelerating the learning process (see Section \ref{sec:synth_data} for details).

\subsection{Proposed Network Architecture and Algorithm}
\label{sec:method}
\textbf{A Need for Memory-Based Model.} Unlike the supervised case where classification label for a datapoint can be determined by its feature values alone, cluster identity for a datapoint depends solely on the identity assigned to its neighbors (previously seen data points). This necessitates using a memory based clustering model. 
We choose an LSTM network \citep{hochreiter1997long} to capture long range dependencies as we do training and testing sequentially.
Apart from storing sequential information over extended time intervals, LSTMs are also flexible enough to learn what kind of information should be passed or blocked for effective prediction (see  \citealp{hochreiter1997long} for details). LSTMs have also been shown to give good accuracies in supervised meta learning \citep{DBLP:conf/icml/SantoroBBWL16}. Here we apply them for the unsupervised meta-clustering case.

At each time step $t$, our LSTM module takes in a datapoint $\mathbf{x}$ and a score vector $\mathbf{a}_{t-1}$ from previous time step ${t-1}$
and outputs a new score $\mathbf{a}_t$ for the current timestep. The score vector encodes the quality the predicted label assigned to the datapoint $\mathbf{x}$. 

\begin{figure}[t]
\centering
  \includegraphics[width=0.5\textwidth]{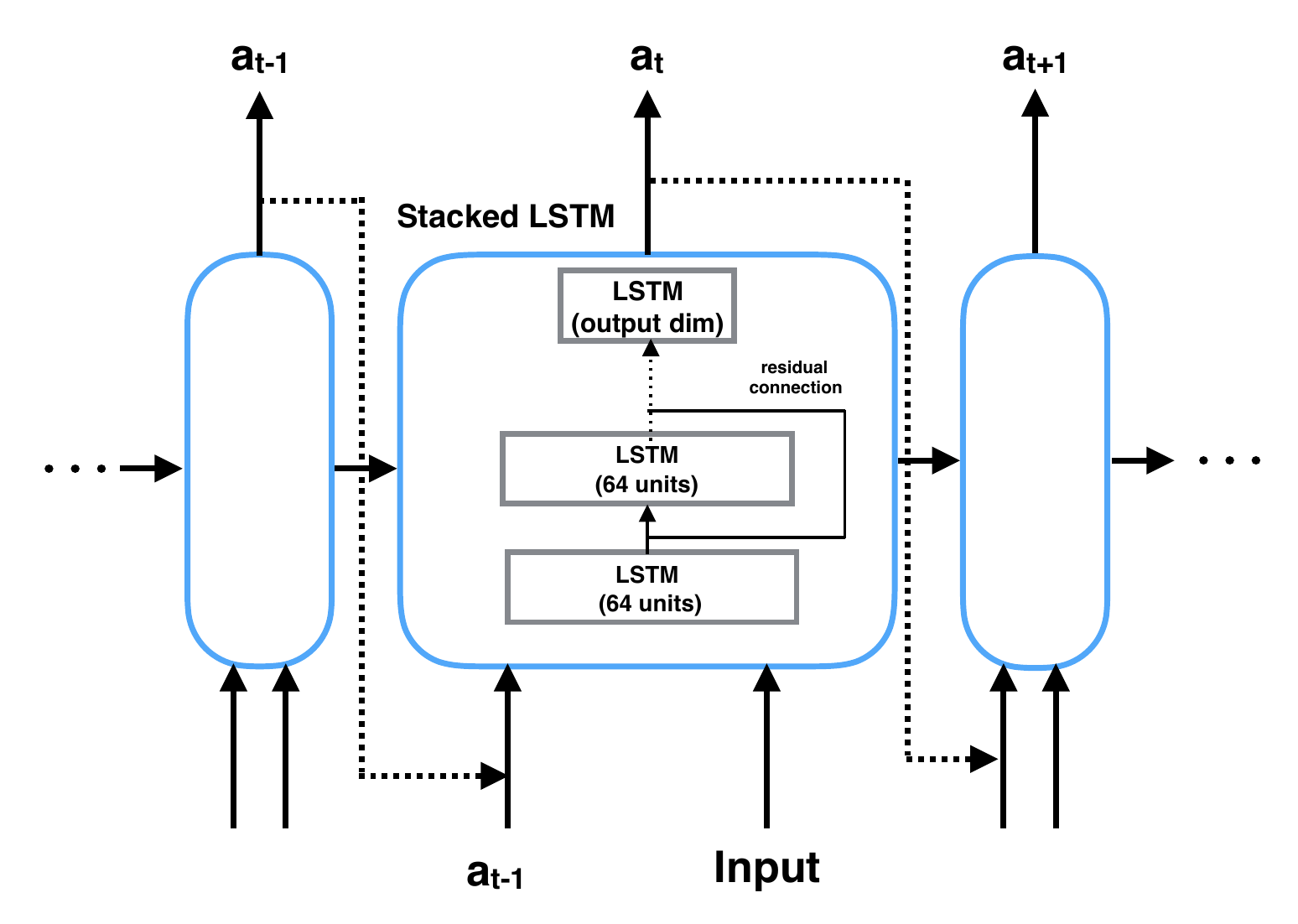}
  \caption{LSTM architecture of meta-learning for clustering.} 
  \label{fig:architecture}
\end{figure}

In this work, we use four LSTMs stacked on top of each other. The first three layers all have 64 units with residual connections while the last layer can have number of hidden units as either (i) the number of clusters (if the desired number of clusters is known, like in the case of $k$-means), or (ii) the maximum number of possible clusters (if the number of clusters is unknown). 
See Figure \ref{fig:architecture} for more details.

\textbf{Loss Function and Optimization.}
Our network architecture for meta-clustering optimizes for a loss function that is a combination of classification loss $\Lclassify$ (that ensures output labels match the training labels) and a local loss $\Llocal$ (that ensures output labels for nearby datapoints match each other). Specifically,
\begin{align*}
    \mathbf{L}_\textup{meta-cluster}(\Phi) &:=\lambda \Lclassify (\Phi) +(1-\lambda)\Llocal (\Phi),
\end{align*}
where $\Phi$ denotes the parameters of our architecture, $\lambda$ is a hyper-parameter controlling the trade-off between the two losses.

Learning the model parameters for the task $\mathcal{T}_i$ proceeds as follows. Let $\mathbf{x}_j^i \in \mathcal{X}^i$ be a datapoint with ground truth cluster label $s_j^i\in\{1,2,\ldots,K_i\}$, and let $\mathbf{r}_j^i :=[r_j^{i,1}, \ldots , r_j^{i,K}]$ be the predicted score vector (i.e.\ the normalized vector $\mathbf{a}$) returned by our model (here $K$ denotes the size of the output layer), and the individual components $r_j^{i,k}$ represent the predicted probability of $\mathbf{x}_j^i$ belonging to cluster $k$. One can also view $\mathbf{r}_j^i$ as a soft assignment of the datapoint to the $K$ clusters. 

Using these definitions, we define $\Lclassify$ and $\Llocal$ as 
\begin{align*}
    \Lclassify&:=-\sum_{j\textrm{th point in } \mathcal{X}^i}{\log(r_j^{i, s_j^i})}, \\
    \Llocal&:=\sum_{j\textrm{th point in } \mathcal{X}^i} {\sum_{j' \in \mathcal{N}(j)} \textup{KL}(\mathbf{r}_j \|\mathbf{r}_{j'})}.
\end{align*}
Note that $r_j^{i,s_j^i}\in \R$ is simply the predicted probability of $\mathbf{x}_j^i$ belonging to its ground truth cluster $s_j^i$, and $\mathcal{N}(j)$ denotes the nearest neighbors of $j$th datapoint. We chose the number of neighbors as 3 for all our experiments, which can be tuned further by cross-validation.

We train our model using Adam optimizer \citep{Kingma2014AdamAM}. Note that typical optimizers used for meta-learning, such as MAML \citep{Finn2017ModelAgnosticMF} and Reptile \citep{Nichol2018OnFM}, are specifically geared towards optimizing for supervised meta-learning tasks, and do not perform as well in our case. This is likely due to the sequential nature of our model.

The overall training and testing (across all training and testing tasks) is performed as follows.

\begin{algorithm}
\caption{Meta-Clustering (Training)}
\label{alg:meta-clustering}
\begin{algorithmic}[1]
\STATE Initialize model parameters $\Phi$ 
\FOR {iteration $i=1,2,3,\ldots$}
    \STATE {Clear LSTM cell states}
    \STATE {Randomly sample $N$ (batchsize) training datasets $\{(\mathcal{X}^{i_1}, \mathcal{L}^{i_1}),\ldots, (\mathcal{X}^{i_N},\mathcal{L}^{i_N})\}$}
\FOR{epochs $1,2,3,\ldots$}
    \STATE {Shuffle each sampled dataset $(\mathcal{X}^i, \mathcal{L}^i)$}
    \STATE {Sequentially update $\Phi$ via Adam Opt. }
\ENDFOR
\ENDFOR
\STATE \textbf{return} $\Phi$
\end{algorithmic}
\end{algorithm}

\begin{algorithm}
\caption{Meta-Clustering (Testing)}
\label{alg:meta-clustering-test}
\begin{algorithmic}[1]
\STATE {Use the learned $\Phi$ from training}
\FOR{each testing dataset $\mathcal{X}$}
\STATE {Clear LSTM cell states}
    \FOR{epochs $1,2,3,\ldots$}
        \STATE {Shuffle test dataset $\mathcal{X}$}
        \STATE {Sequentially predict the cluster id}
    \ENDFOR
\ENDFOR
\STATE \textbf{return} the corresponding clusterings 
\end{algorithmic}
\end{algorithm}

Observe that during each iteration in training, we randomly sample $N$ (batchsize) training datasets from our given pool of training tasks. The datasets as well as their ground truth labels (for optimization) are fed into our LSTM network architecture sequentially. The LSTM cell states are kept across epochs. This enables the network to keep memory of previously seen data points.

During the testing phase, each test task is fed into the pre-trained network to obtain the clustering. 
It is important to note that datapoints in each dataset are shuffled across iterations during both the training and testing phases. This prevents potential prediction errors introduced by specific sequence orders.

\subsection{The Possibility of Meta-Clustering}
Kleinberg's impossibility theorem \citep{DBLP:conf/nips/Kleinberg02} states that clustering is impossible because there is no clustering algorithm that can simultaneously satisfy three very intuitive axioms that any clustering algorithm should follow (the so called axioms of \textit{Scale-Invariance}, \textit{Richness}, and \textit{Consistency}). Interestingly, for \emph{meta clustering} one can formulate a generalized set of axioms that are in fact  consistent showing the \emph{possibility} of meta-clustering in a specific framework as detailed by  \citet{DBLP:conf/nips/Garg18}. Since our framework is different from \citet{DBLP:conf/nips/Garg18} (cf.\ Section \ref{sec:related}), their result is not directly applicable.
Luckily we can, however, formulate an alternate set of axioms that are more akin to our framework and achieve consistency. In this section, we will discuss how we can reframe the three axioms for meta-clustering to circumvent Kleinberg's impossibility result. 

Consider a finite set of points $X$ and the class of 
all possible symmetric distance distance functions $D(X)$ on $X$.
A clustering algorithm $\mathcal{A}$ can be viewed as taking a distance function $d \in D(X)$ as an input and returning a partition---i.e., a clustering---of $X$. With this notation, Kleinberg's clustering axioms for any clustering algorithm $\mathcal{A}$ can be stated as follows. 
%
%
\begin{itemize}
    \item 
    \textbf{Scale-Invariance}. For any distance function $d$ and any $\alpha > 0$, $\mathcal{A}(d) = \mathcal{A}(\alpha \cdot d)$.
    
    \item
    \textbf{Richness.} For any finite $X$ and clustering $\mathcal{C}$ of $X$, there exists a distance function $d \in D(X)$ such that $\mathcal{A}(d) = \mathcal{C}$.
    
    \item
    \textbf{Consistency.} Let $\mathcal{C}$ be the clustering produced by some distance function $d \in D(X)$, that is $\mathcal{A}(d) = \mathcal{C}$. Consider any distance function $d'\in D(X)$, such that for all  $x, \bar x \in X$, if $x, \bar x$ are in the same cluster in $\mathcal{C}$ then $d'(x,\bar x) \leq d(x,\bar x)$, and if $x,\bar x$ are in different clusters in $\mathcal{C}$ then $d'(x,\bar x) \geq d(x,\bar x)$. Then it must be the case that $\mathcal{A}(d') = \mathcal{A}(d)$.
    
\end{itemize}
\citet{DBLP:conf/nips/Garg18} suggests a re-framing of these axioms for \emph{meta}-clustering. Specifically, by introducing a variant of Scale-Invariance, \citet{DBLP:conf/nips/Garg18} shows that there is a meta-clustering algorithm that satisfies the new Scale-Invariance axiom and whose output always satisfies Richness and Consistency. Different from their version, we consider a formulation that is more appropriate and applicable in our setting. 

Suppose $\mathcal{M}$ is a meta-clustering algorithm as described in Section~\ref{sec:meta-setup}. We can view $\mathcal{M}$ to consist of two steps. First, $\mathcal{M}$ takes a distance function $d$ on the input dataset $X$ and outputs a \emph{clustering algorithm} $\mathcal{A}$ (instead of a clustering), i.e.\ $\mathcal{M}[d] = \mathcal{A}$. Second, the clustering algorithm $\mathcal{A}$, in turn, will do the clustering via $d$ and outputs a partition $\mathcal{C}$, i.e. $\mathcal{A}(d) = \mathcal{C}$. Essentially, $\mathcal{M}[d](d) = \mathcal{C}$. $\mathcal{M}$ is trained in the meta-training phase to adapt to clustering tasks, provided with datasets and true clustering labels, unlike \citet{DBLP:conf/nips/Garg18} which returns one clustering algorithm based on training datasets and labels. In our case, the LSTM architecture does both the steps, performing clustering on input data $X$ but also adapting to $X$. The adaptation happens through the change of activations and gates' values inside the LSTM. But how the activations and gates' values of LSTM are changed for different input data is determined by the LSTM weights which are trained during meta-training and are fixed through meta-testing. 
The meta version of the axioms in our setting are as follows.
\begin{itemize}
    \item 
    \textbf{Meta-Scale-Invariance.} For any $\alpha > 0$, $\mathcal{M}[d](d) = \mathcal{M}[\alpha d](\alpha d)$.
    \item
     \textbf{Meta-Richness.} For any finite $X$ and clustering $C$ of $X$, there exists $d \in D(X)$ such that $\mathcal{M}[d](d) = \mathcal{C}$.
    \item
    \textbf{Meta-Consistency.} 
    Let $\mathcal{C}$ be the clustering produced by some distance function $d \in D(X)$, that is $\mathcal{M}[d](d) = \mathcal{C}$. Consider any distance function $d'\in D(X)$, such that for all  $x, \bar x \in X$, if $x, \bar x$ are in the same cluster in $\mathcal{C}$ then $d'(x,\bar x) \leq d(x,\bar x)$, and if $x,\bar x$ are in different clusters in $\mathcal{C}$ then $d'(x,\bar x) \geq d(x,\bar x)$. Then it must be the case that $\mathcal{M}[d](d') = \mathcal{M}[d](d)$.
\end{itemize}
    The original consistency axiom requires $\mathcal{A}(d) = \mathcal{A}(d')$, where $d'$ can shrink intracluster and expand intercluster distances. 
    The impossibility arises from the fact that the distance distortion (shrinking and expansion) of $d'$ could have been from any distance function either the original $d$ or some other $d''$ that may produce a different clustering. 
    By introducing a level of indirection through meta-learning, we can essentially specify which distance function ($d$ or some other $d''$) the clustering algorithm uses, thus resolving the conflict. In fact the following holds true.

\begin{theorem}
There is a meta clustering algorithm that satisfies Meta-Scale-Invariant and Meta-Richness, Meta-Consistency for $|X| > 2$.
\end{theorem}
\begin{proof}
Similar to \cite{DBLP:conf/nips/Kleinberg02}, let's consider the family of single-linkage clustering functions. Each single-linkage clustering function has a threshold $\lambda$. The single-linkage clustering function operates by initializing each point as its own cluster and adding edges between two points if their distance is below $\lambda$. At the end, the connected components of the output graph are the clusters. 

We can construct a meta-clustering algorithm as a following two step procedure. Let $\frac{1}{2} < \tau < 1$ be a fixed constant. Now given any dataset $X$ (such that $|X|>2$).
\begin{itemize}
    \item (Meta-step, i.e.\ the choosing the clustering algorithm $\mathcal{A}$) Given a distance function $d$ on $|X|$, we can pick a single-linkage clustering function choosing $\lambda:= \tau \rho^*$, where $\rho^{*} = \max_{i,j} d(i, j)$. Observe that this step is akin to $\mathcal{M}[d]$, basically, $\mathcal{M}[d] = \mathcal{A}_\lambda$ (for the specified $\lambda$).
    \item (Clustering-step, i.e.\ applying $\mathcal{A}$ to the dataset for clustering) Once $\mathcal{A}_\lambda$ is picked, we can run single-linkage (with threshold $\lambda$) on $X$ with distance function $d$. This step is akin to $\mathcal{A}(d)$.
\end{itemize}

We will show that the meta-clustering algorithm as described above satisfy the three meta-clustering axioms.

\textbf{Meta-Scale-Invariance.} Suppose $\mathcal{M}[d](d) = \mathcal{C}$, or equivalently, $\mathcal{A}_\lambda(d) = \mathcal{C}$ ($\lambda$ as defined above). By the construction of the meta clustering algorithm, $\mathcal{M}[\alpha d]$ returns a single linkage function $\mathcal{A}_{\alpha\lambda}$. It is trivial to show that $\mathcal{A}_\lambda(d) = \mathcal{A}_{\alpha\lambda}(\alpha d)$, that is, the scaling gives back exactly the same clustering $\mathcal{C}$. 

\textbf{Meta-Richness.} Consider an arbitrary partition $\mathcal{C}$ with at least two clusters. We can construct the following distance function $d$ such that if $x, \bar x$ are in the same cluster in $\mathcal{C}$, $d(x, \bar x) = 1$ and $d(x, \bar x) = 2$ if not. It is easy to verify that this is a valid distance metric. Since $\frac{1}{2} < \tau < 1$, the threshold $\lambda$ would be strictly larger than $1$. Therefore, meta-clustering can successfully recover $\mathcal{C}$. If $\mathcal{C}$ only has one cluster, we arbitrarily select a pair of data points $x_i$ and $x_j$ and set $d(x_i, x_j) = 2$ with the rest of pairwise distances to be $1$. This is still a valid metric. There will be no edge between $x_i$ and $x_j$ while the other nodes are connected. Because $|X| > 2$, we know there exits at least another node $x_z$ such that $x_z$ and $x_i$ is connected and $x_z$ and $x_j$ is connected. Therefore, there is still only one connected component and thus one cluster.

\textbf{Meta-Scale-Consistency.} Suppose $\mathcal{M}[d]$ returns a single-linkage clustering function with threshold $\lambda$, that is, $\mathcal{A}_\lambda$. Now consider applying this clustering function to a different distance function $d'$ defined as follows. Let $\mathcal{C}$ be the clustering produced by $d$, that is $\mathcal{C}=\mathcal{A}_\lambda$. If $x, \bar x$ are in the same cluster in $\mathcal{C}$ then $d'(x,\bar x) \leq d(x,\bar x) \leq \lambda$. And, if $x,\bar x$ are in different clusters in $\mathcal{C}$ then $d'(x,\bar x) \geq d(x,\bar x) \geq \lambda$. Observe that the single-linkage clustering function with threshold $\lambda$ will create the same graph for $d$ and $d'$ and thus $\mathcal{A}_\lambda(d) = \mathcal{A}_\lambda(d')$, producing the same clusters. 
\end{proof}

\section{Experiments}
\label{sec:experiments}

We evaluate the efficacy of our proposed LSTM network architecture for Meta-Clustering by performing various synthetic and real-world experiments that measures how various aspects of the input data (such as representation dimension, number of true clusters, etc.) affect the prediction quality.

We start by describing our synthetic data generation process.

\subsection{Synthetic Dataset Generation}
\label{sec:synth_data}
\textbf{Generating Gaussian-shaped clusters.} The most basic kind of clustering are those where each cluster is generated from a multivariate Gaussian distribution with a random mean and covariance. Specifically, to generate $K$ clusters in $d$ dimensions, we sample the means $\mu_i$ and covariances $\Sigma_i$ ($i$ denotes the cluster id) as
\begin{align*}
\mu_i &\sim \textup{Uniform}([\theta,\theta]^d), 
\;\;\;\;\; \Sigma_i  := \hat{C}_i^\mathsf{T}\hat{C}_i, 
\;\;\; \textrm{ where} \\
\hat{C}_i &= \textup{orth}(C_i)/ k_i, \,
C_i :=[c_1 \ldots c_d ], \\
c_j &\sim \textup{Gaussian}(0,I), \,
k_i \sim \textup{Uniform}(\alpha,\alpha+\beta).
\end{align*}

$\textup{orth}(\cdot)$ denotes the orthogonalization of $\cdot$, and parameters $\alpha$ and $\beta$ control the magnitude of the entries in the covariance matrix. Smaller $\alpha$ and $\beta$ results in overlapping clusters, and thus results in harder to distinguish clusters.

\textbf{Generating curved shaped clusters by adding simple nonlinearities.}
Mere Gaussian shaped data generation may not capture the complex nature of some real world data.
We therefore introduce simple nonlinearities to our synthetically generated random Gaussian clusters. For any point $x$ and two arbitrary coordinates $p,q$, we apply the following:
\begin{align*}
&\hat{x}_p = \cos(t)x_p + \sin(t)x_q, \;
\hat{x}_q = -\sin(t)x_p + \cos(t)x_q \\
 &\textrm{where } t = \pi r/k_r, \, r=\|z\|, \, z = [x_p,x_q ], \\
&k_r  \sim \textup{Uniform}(\alpha_r,\alpha_r+\beta_r).
\end{align*}

We apply this transformation to points in a cluster several times (each time selecting two different coordinates). 
See Figure~\ref{fig:sample_data} for example datasets generated by this process.

\begin{figure}
\centering
  \begin{tabular}{@{}cc@{}}
   \includegraphics[width=.25\textwidth]{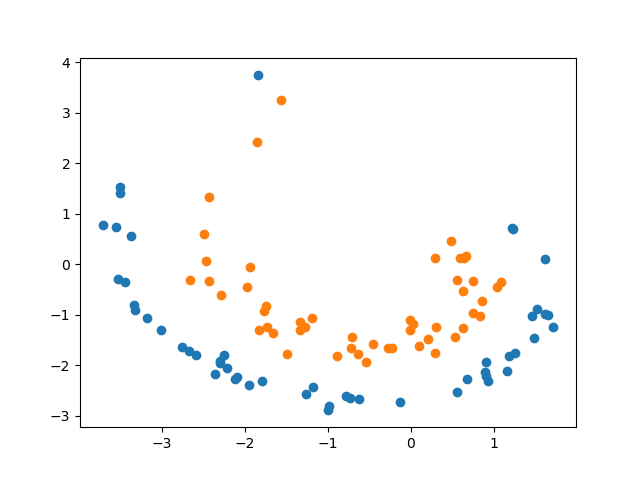} 
     \includegraphics[width=.25\textwidth]{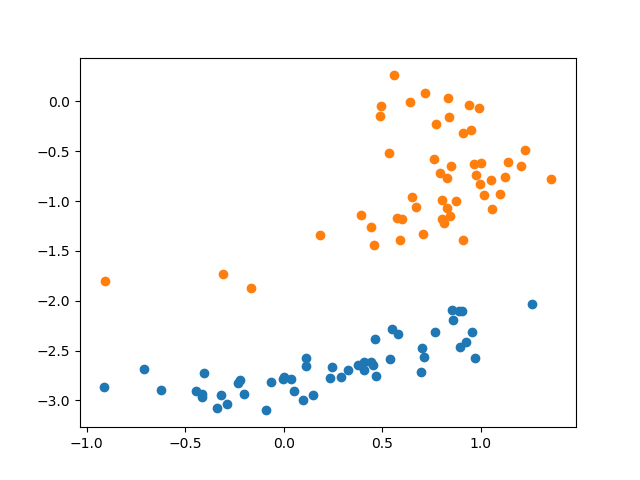}  \\
    \includegraphics[width=.25\textwidth]{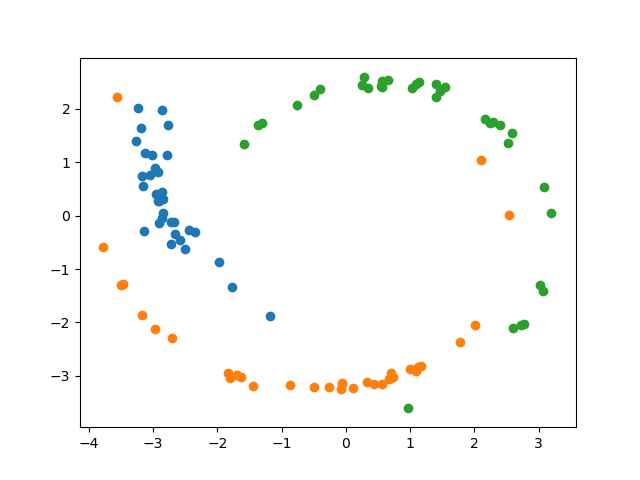} 
    \includegraphics[width=.25\textwidth]{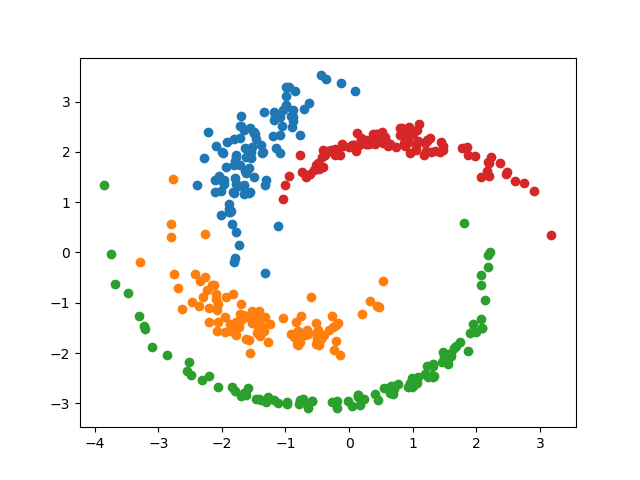}
    \\
  \end{tabular}
  \caption{Sample synthetic datasets generated by our procedure.}
  \label{fig:sample_data}
\end{figure}

\textbf{Assigning the cluster identity.}
While we can assign labels to the randomly generated clusters with any permutation, the extra degree of freedom makes the training harder. Instead, the clusters are sorted by the first dimension of their mean vectors and the cluster ids are assigned sequentially. 



\subsection{Baseline Methods and Evaluation Criteria}
We use several popular clustering methods as baselines to compare the quality of our proposed Meta-Clustering.
\begin{itemize}
\item $k$-means -- it is arguably the most popular clustering method; the number of clusters needs to be prespecified; can only find convex clusters. 
\item Kernelized $k$-means -- it is a nonlinear extension to $k$-means that can potentially find arbitrary shaped clusters. We use radial basis function (rbf) kernel for all our experiments.
\item Spectral Clustering  -- it is another very popular nonlinear extension to $k$-means that can potentially find arbitrary shaped clusters. 
\item DBSCAN -- it is a density based clustering that can find arbitrary shaped clusters, and is robust to noise. The user does not need to specity the number of clusters. It is most effective on low-dimensional clustering problems.
\item {DEC} -- it is a deep learning based clustering algorithm \citep{xie2016unsupervised}. It employs an auto-encoder as feature extractor and uses soft assignment to calculate loss to optimize. Because auto-encoder needs to be trained for each dataset, the runtime can be long when presented with samples from multiple datasets and it is most effective when dataset is large. 
\end{itemize} 
The reported results use the best parameter settings for each of the baseline methods.  The clustering quality is evaluated using the 0-1 loss.
Since any permutation of labels should not change the clustering quality, the best permutation of the predictions also needs to be taken into account 
using the  
Kuhn-Munkres algorithm.



\begin{figure*}[b]
\centering
\begin{subfigure}
    \centering
  \begin{tabular}{@{}cccc@{}}
    \includegraphics[width=0.3\textwidth]{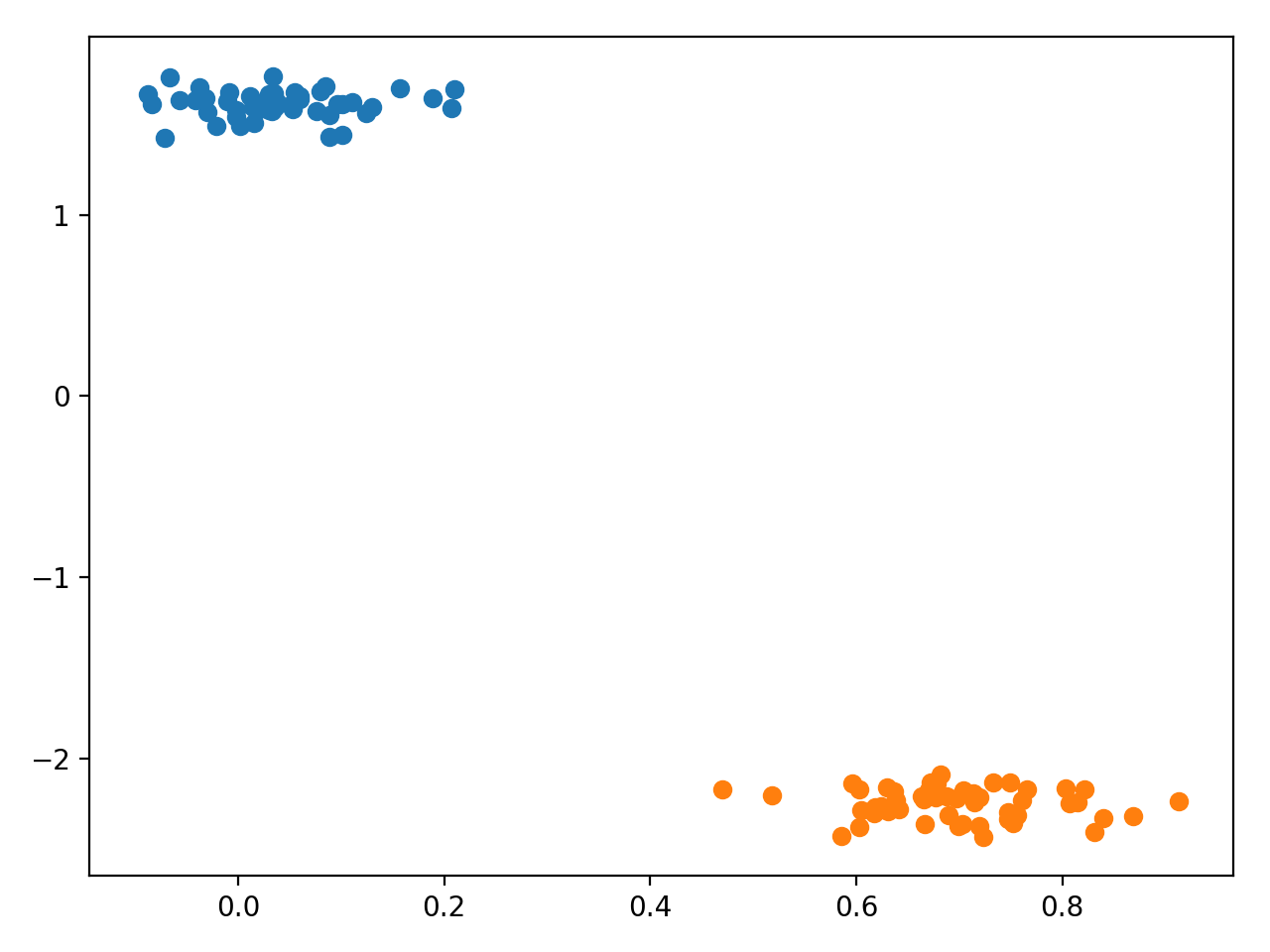} &
    \includegraphics[width=0.3\textwidth]{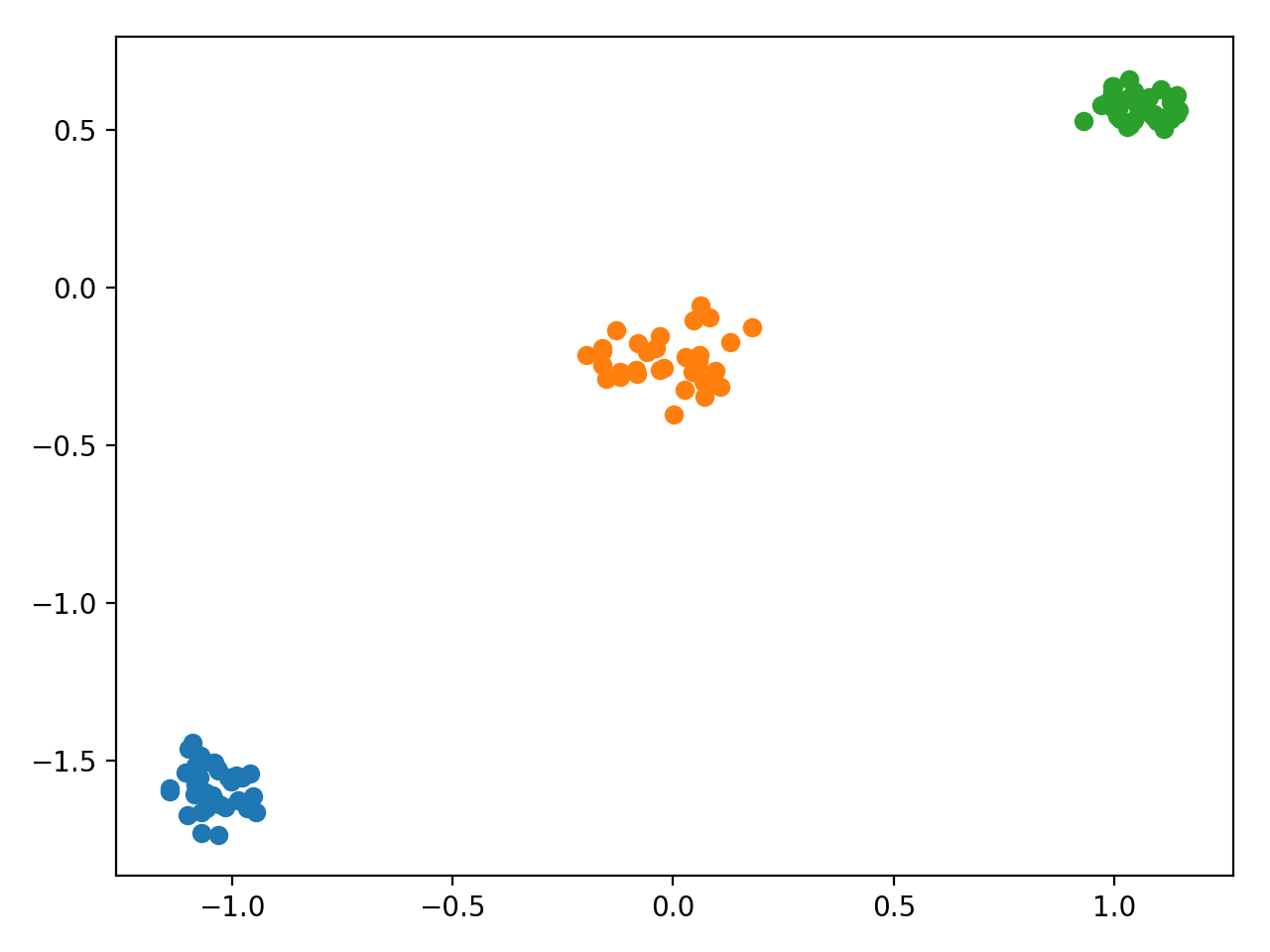} &
     \includegraphics[width=0.3\textwidth]{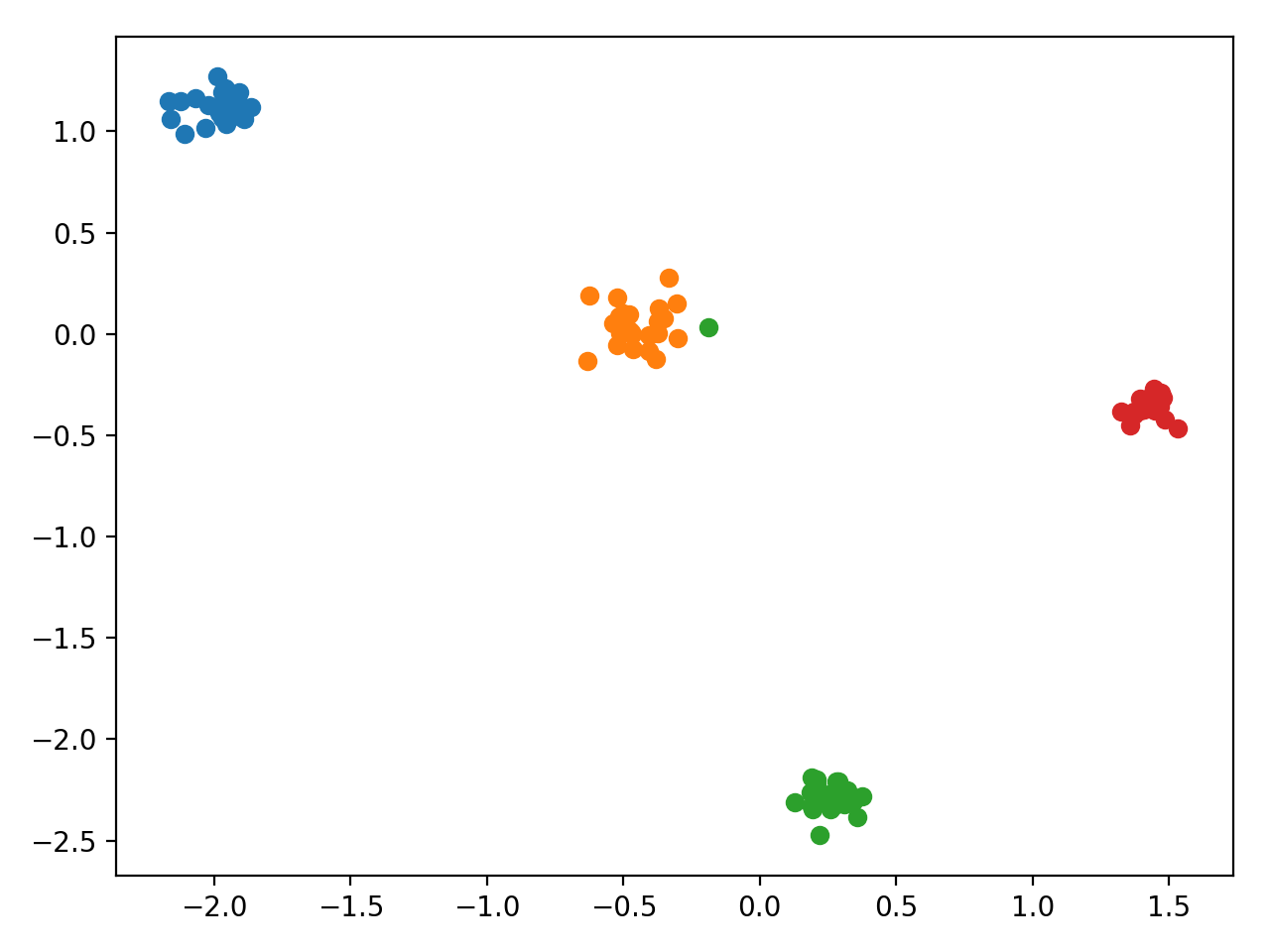} 
  \end{tabular}
  \caption{Clustering results when varying number of clusters in the test task.}
  \label{fig:adaptive}
\end{subfigure}

\end{figure*}
\subsection{Experiments on Synthetic Datasets}
\label{sec:syn-exp}
We first evaluate our model on synthetic test datasets by training on synthetic datasets. Training and test datasets are generated the same way but drawn independently.

\begin{table*}
    \centering
    \caption{Average Error rate of meta-clustering on synthetic dataset 
    (100 samples).}
    \begin{tabular}{ccccccc}
        \toprule
        $k$-means & \tabsplit{Kernel \\ $k$-means} & \tabsplit{Spectral \\ Clustering} & DBSCAN  & DEC & Meta-Clustering & \tabsplit{Meta-Clustering \\ (scaled)}\\
        \midrule
        0.17	& 0.29 & 0.17 & 0.48 & 0.18 & \textbf{0.08} & 0.10\\
        \bottomrule
    \end{tabular}
    \label{tab:2-D}
\end{table*}

\textbf{2-D case study.}
Table~\ref{tab:2-D} shows the performance of our meta-clustering model when compared to the other benchmark clustering methods. The model is trained on synthetic dataset with two clusters. Since the datasets generated are arbitrary shaped clusters (and not necessarily convex), centroid based methods that operate in the input representation like $k$-means is not expected to work well. Remarkably even more flexible methods such as spectral clustering do not yield significant improvement over $k$-means either. Our meta-clustering method (second to last column in Table~\ref{tab:2-D}) performs the best, demonstrating the power of the meta-learning framework: when trained on similar related tasks. Meta-clustering in some sense \emph{learns} the right notion of cluster loss and can outperform even the most popular cluster losses for simplest of tasks.


\textbf{Location and Scale Invariance.}
Meta clustering can also do well on datasets that are not limited to where the training data resides in the representation space. In this experiment, our data points during test were translated and scalded by a factor of $3$. As shown in the last column of Table~\ref{tab:2-D}, meta-clustering still performs comparably. This indirectly suggests that meta-clustering does not cluster simply based on what is observed in the training data, and generalizes well. This observation is further corroborated in Section \ref{sec:real_data} by testing the synthetically trained model on real datasets and Section \ref{sec:open-ml} by training and testing on real datasets of different distributions.


\textbf{Adapting to the Number of Clusters.}
It is worth noting that our model has the ability to approximate the number of clusters in the new clustering task. 
Though, there is still a limit on the maximum number of clusters (determined by the output dimension of our LSTM architecture). 
In this experiment, our model is constrained to output up to 5 clusters. The training data consists of synthetically generated clusters with the number of clusters varying between 1 and 5. We sample more training datasets with higher number of clusters to prevent training models that are biased towards returning fewer clusters.
Figure~\ref{fig:adaptive} shows the results on test datasets, showing the ability of our meta-clustering architecture to adapt to given dataset.

\subsection{Training with Synthetic Data}
\label{sec:real_data}
We also evaluate our method on several real datasets. Our primary goal is to evaluate how our clustering algorithm performs on real datasets even when the training is done on synthetic datasets. Training configurations (number of clusters, feature dimensions, number of data points) for synthetic data will match the test case.

\begin{table*}
    \centering
    \caption{Error rates on UKM, MNIST and IRIS datasets (trained on synthetic data)
    }

    \begin{tabular}{ccccccc}
        \toprule
        Method & $k$-means & Ker. $k$-means & Spec. Clust. & DBSCAN  & DEC & Meta Clust. \\
        \midrule
        UKM & 0.31	&\textbf{0.26}	&0.28	&0.49 & 0.28	&\textbf{0.26} \\
        \midrule
        MNIST & 0.35	& 0.33	& 0.40	& 0.49	& 0.46 & \textbf{0.32} \\
        \midrule
        IRS & 0.19	&0.23	&0.19	&0.38 & 0.24 & \textbf{0.09} \\
        \bottomrule
    \end{tabular}
    \label{tab:MNIST}
\end{table*}

\textit{User Knowledge Modeling}
\citep{Kahraman:2013:DIK:2400768.2401504} is a dataset about the students' knowledge on the subject of Electrical DC Machines. Each example has five attributes describing various aspects of a student's knowledge. Students are classified into four knowledge levels: ``very low", ``low", ``medium" and ``high".

During training, we randomly generate 100 synthetic datasets per epoch and train for 50 epochs. Each synthetic datasets consists of 100 data points and 2 clusters. The generation process is the same as described in Section \ref{sec:synth_data}. For test, we sample 100 points from ``low" and ``high" classes, and evaluate the trained model by averaging the error rates of 100 runs. 
The result is shown in Table \ref{tab:MNIST}. Notice that meta-clustering gets competitive error rates. 


\textit{MNIST database} \citep{726791} of handwritten digits has 70,000 examples. 
We preprocess the dataset by applying PCA down to 2 dimensions. 
{We trained the model on synthetic datasets the same way as described for UKM dataset} but with 1000 data points per dataset. During test, we sample 1000 data points from two randomly selected two digits each time and the error rate shown in Table~\ref{tab:MNIST} is the average over 100 samples.
Table~\ref{tab:MNIST} shows that meta clustering outperforms other standard clustering methods. For a fair comparison, the auto-encoder in DEC is re-trained for each sample because each clustering algorithm should only look at the sample given. 

\textit{Iris dataset} \citep{Dua:2017} contains three classes with a total of 150 data points and 4 features. The model is trained similarly as before with 150 data points per synthetic dataset and tested on the 150 iris data points without sampling.
Table~\ref{tab:MNIST} shows that our model performs much better than the standard benchmarks. 
Figure~\ref{fig:iris} shows the clustering result on this data in detail. Notice that our method (right plot) can uncover the two clusters on the right much better than $k$-means (left plot).

\begin{figure*}
\setlength\tabcolsep{2pt}
\centering
  \begin{tabular}{@{}ccc@{}}
    \includegraphics[width=.33\textwidth]{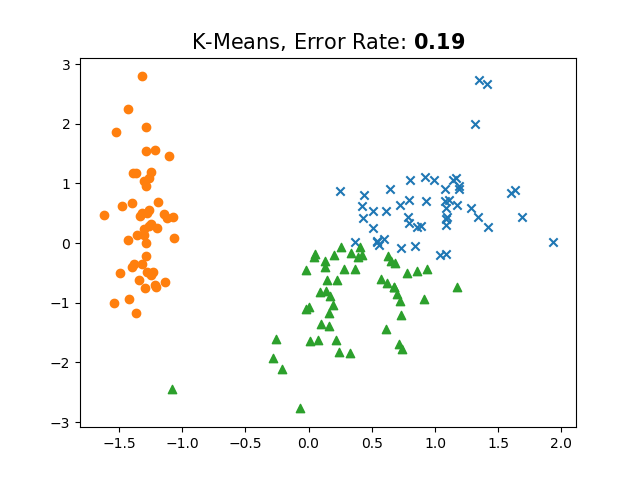}&
    \includegraphics[width=.33\textwidth]{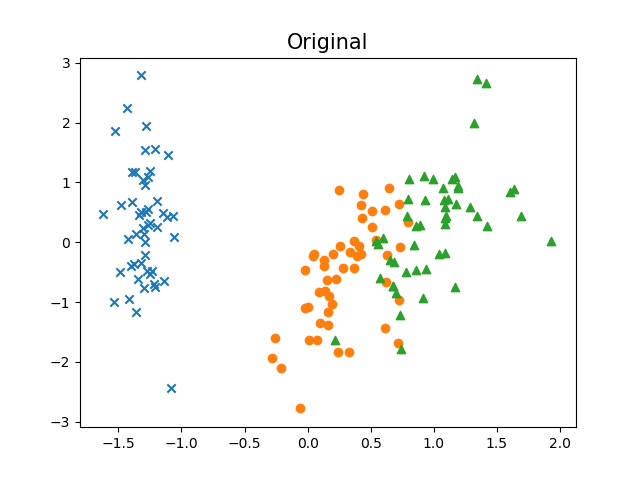} &
   \includegraphics[width=.33\textwidth]{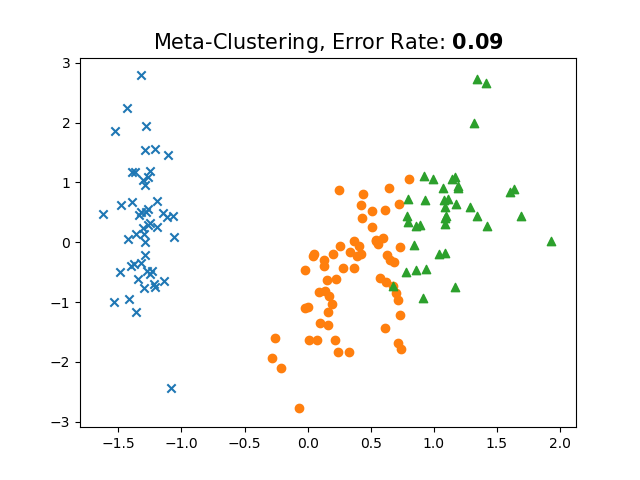} 
  \end{tabular}
  \caption{ \footnotesize Visualized comparison of $k$-means and meta-clustering on Iris dataset. Visualization made by projecting the data onto the top two principal component. Center: Iris data with ground truth labelling, Left: clustering produced by $k$-means, Right: clustering produced by meta-clustering. Quantitatively, error rate measure shows Meta-clustering produces significantly better quality clusters.}
  \label{fig:iris}
\end{figure*}

\subsection{Training with Real Data}
\label{sec:open-ml}

Following \cite{DBLP:conf/nips/Garg18}, we also train our model the on the repository of classfication datasets from openml \citep{OpenML2013} database. By excluding labels, they can be viewed as clustering problems. For each experiment, we fetch all the datasets in openml repository that satisfied the desired feature dimensions and number of clusters (classes). We then randomly selected $10\%$ of the queried datasets for test and the rest for training. To emphasize the generalization power of our model, we do not tune the hyperparamters for each experiment and instead keep the architecture same for all the experiments. We randomly sample datapoints from each dataset to form meta-training and meta-test datasets. For each experiment, every meta-training or meta-test dataset has the same number of datapoints. To avoid imbalanced clusters dominating meta-training, we choose to sample each cluster uniformly (similar to \citealp{Hsu2018UnsupervisedLV}). 

Table~\ref{tab:openml} shows the results when varying $k$ (the number of clusters), $N$ (the number of data points per meta dataset) and the range of feature dimensions. The feature dimension range is chosen such that the average number of datasets per feature dimension is high for effective training. We padded every dataset with zeros such that they all have the same dimension in each case. Results show that our Meta Clustering method typically has the lowest error rate. 

Note that in some cases, there may not be enough datasets for training or the available datasets are not diverse enough or inherently hard to cluster. Unlike other deep clustering models, our model is only designed to learn to cluster so the performance can be limited by raw features. But as shown in the previous section, training on synthetic dataset can be helpful to clustering real data. Therefore, we augment the openml datasets with synthetic datasets for models that are under-trained by the available openml datasets, specifically, the experiments with $k=3$ (28 datasets available) and $k=4$ (20 datasets available). For $k=3$, the newly trained model significantly out-performs other baselines. For $k=4$, the error rate dropped getting closer to the best clustering algorithm in this experiment. Observe that it does not surpass every benchmark partially because the generalization ability of the model is limited by the training datasets; poorly chosen training datasets can be detrimental to meta clustering.

We also explore the case when the number of clusters is not fixed but rather chosen from a range. For methods like $k$-means, the choice of $k$ can be ambiguous. While there are approaches like Elbow methods to choose $k$, such heuristics are hard to apply across different datasets in the openml repository. So for clustering algorithm that requires $k$, we use the maximum number of possible clusters. This is also a fair comparison because the output dimension of our meta model is also fixed in the same way as described in Section~\ref{sec:method}. Table~\ref{tab:openml_multi_k} demonstrates that meta-clustering outperforms in this case as well (including the deep clustering DEC benchmark). 

\begin{table*}[h]
    \centering
    \setlength\tabcolsep{2pt}
    \caption{Error rate on openml test datasets (100 samples). The second error rate of the Meta clustering column, if needed, comes from models trained with real datasets and synthetic datasets. 
    }. 
    \begin{tabular}{ccccccc}
        \toprule
        Method & $k$-means & Ker. $k$-means & Spec. Clust. & DBSCAN  & DEC & Meta Clust. \\
        \midrule
        \tabsplit{k = 2, N = 100 \\ dims: [1, 15]} & $\textbf{0.38} \pm \textbf{0.09}$	& $0.43 \pm 0.07$ & $0.42 \pm 0.10$ & $0.51 \pm 0.05$ & $0.40 \pm 0.09$ & \tabsplit{$\textbf{0.38} \pm \textbf{0.09}$ \\ \NA} \\
        \midrule
        \tabsplit{k = 3, N = 100\\ dims: [1, 15]} & $0.11 \pm 0.02$ & $0.13 \pm 0.09$	& $0.12 \pm 0.05$	& $0.29 \pm 0.03$	&$0.22 \pm 0.14$ &  \tabsplit{$0.34 \pm 0.01$ \\ $\textbf{0.02} \pm \textbf{0.01}$ }\\
        \midrule
         \tabsplit{k = 4, N = 500 \\ dims: [1, 20] } & $\textbf{0.54} \pm \textbf{0.03}$	& $0.62 \pm 0.08$	& $0.56 \pm 0.05$	& $0.74 \pm 0.01$	& $0.57 \pm 0.08$& \tabsplit{$0.63 \pm 0.02$  \\ $0.58 \pm 0.02$ }\\
         \midrule
          \tabsplit{k = 6, N = 500 \\ dims: [1, 40]} & $0.76 \pm 0.01$	& $0.79 \pm 0.01$	& $0.82 \pm 0.00$	& $0.83 \pm 0.00$	& $0.73 \pm 0.09$ & \tabsplit{$\textbf{0.59} \pm \textbf{0.04}$\\ \NA }\\
        \bottomrule
    \end{tabular}
    \label{tab:openml}
\end{table*}

\begin{table*}
    \centering
    \setlength\tabcolsep{1.5pt}
    \caption{Error rate on openml test datasets with unknown $k$ (100 samples). 
    }
    \begin{tabular}{cccccccc}
        \toprule
        Method & $k$-means & Ker. $k$-means & Spec. Clust. & DBSCAN  & DEC & Meta Clust. \\
        \midrule
        \tabsplit{k = $\{ 2,3,4 \}$ N: 500 \\ dims: [1, 15]} & $0.60 \pm 0.08$	& $0.65 \pm 0.08$ & $0.57 \pm 0.11$ & $0.56 \pm 0.09$ & $0.54 \pm 0.09$ &$\textbf{0.53} \pm \textbf{0.08}$\\
        \midrule
        \tabsplit{k = $\{ 3,4,5 \}$, N: 500\\ dims: [1, 15]} & $0.62 \pm 0.05$ & $0.67 \pm 0.09$	& $0.59 \pm 0.06$	& $0.65 \pm 0.10$& $0.60 \pm 0.08$	& $\textbf{0.58} \pm \textbf{0.09}$\\
       \bottomrule
    \end{tabular}
    \label{tab:openml_multi_k}
\end{table*}

\section{Discussion and Conclusion}
\label{sec:conclusion}
In this paper, we present a novel deep learning algorithm for learning to cluster. Instead of directly optimizing a specific cluster loss on a given dataset \citep{DBLP:conf/icml/YangFSH17}, we propose to learn a meta-algorithm that can adapt to new clustering tasks. 

Just like the success of any transfer learning problem depends on the similarity between the training and the test tasks, the quality of our meta-clustering is also expected to depend on that. Nevertheless,
we show that our meta-clustering model can perform better than important clustering benchmarks when trained with simple synthetic datasets or on existing labeled datasets.

\bibliographystyle{abbrvnat}

\clearpage
\onecolumn

\end{document}